\documentclass[american]{article}
\usepackage[T1]{fontenc}
\usepackage[utf8]{inputenc}
\usepackage{geometry}
\usepackage{array}
\usepackage{verbatim}
\usepackage{multirow}
\usepackage{varwidth}
\usepackage{amsmath}
\usepackage{amsthm}
\usepackage{amssymb}
\usepackage{graphicx}

\makeatletter

\providecommand{\tabularnewline}{\\}
\newenvironment{cellvarwidth}[1][t]
    {\begin{varwidth}[#1]{\linewidth}}
    {\@finalstrut\@arstrutbox\end{varwidth}}

\newcommand{\lyxaddress}[1]{
	\par {\raggedright #1
	\vspace{1.4em}
	\noindent\par}
}
\theoremstyle{plain}
\newtheorem{thm}{\protect\theoremname}
\theoremstyle{plain}
\newtheorem{lem}[thm]{\protect\lemmaname}
\ifx\proof\undefined
\newenvironment{proof}[1][\protect\proofname]{\par
	\normalfont\topsep6\p@\@plus6\p@\relax
	\trivlist
	\itemindent\parindent
	\item[\hskip\labelsep\scshape #1]\ignorespaces
}{%
	\endtrivlist\@endpefalse
}
\providecommand{\proofname}{Proof}
\fi
\theoremstyle{definition}
\newtheorem{defn}[thm]{\protect\definitionname}
\theoremstyle{plain}
\newtheorem{cor}[thm]{\protect\corollaryname}

\@ifundefined{date}{}{\date{}}
\usepackage{newtxmath}

\makeatother

\usepackage{babel}
\providecommand{\corollaryname}{Corollary}
\providecommand{\definitionname}{Definition}
\providecommand{\lemmaname}{Lemma}
\providecommand{\theoremname}{Theorem}

\begin{document}
\title{Expressing linear equality constraints in feedforward neural networks}
\author{Anand Rangarajan, Pan He, Jaemoon Lee, Tania Banerjee and Sanjay Ranka}
\maketitle

\lyxaddress{\begin{center}
Dept. of CISE\\
University of Florida, Gainesville, FL
\par\end{center}}
\begin{abstract}
We seek to impose linear, equality constraints in feedforward neural
networks. As top layer predictors are usually nonlinear, this is a
difficult task if we wish to use standard convex optimization methods
and strong duality. To overcome this, we introduce a new saddle-point
Lagrangian with auxiliary predictor variables on which constraints
are imposed. Eliminating the auxiliary variables leads to a dual
minimization problem on the Lagrange multipliers introduced to satisfy
the linear constraints. This minimization problem is combined with
the standard learning problem on the weight matrices. From this theoretical
line of development, we obtain the surprising interpretation of Lagrange
parameters as additional, penultimate layer hidden units with fixed
weights stemming from the constraints. Consequently, standard minimization
approaches can be used despite the inclusion of Lagrange parameters---a
very satisfying, albeit unexpected, discovery. 
\end{abstract}

\section{Introduction}\label{sec:Introduction}

The current resurgence of interest and excitement in neural networks,
mainly driven by high performance computing, has tended to shift the
research focus away from applied math development. While this has
enabled a dizzying array of practical applications and has firmly
brought artificial intelligence and machine learning into the mainstream,
fundamental issues remain. One such issue is the enforcement of constraints
in neural network information processing. By far the most common architecture
is the feedforward neural network with its implicit hidden units,
nonlinearities and divergence measures linking predictors to targets.
However, if we seek to impose mathematically specified constraints
on the predictors since the targets satisfy them, there is a paucity
in the literature to address this issue. A survey of the literature
brings out methods that softly impose constraints using penalty functions
and there is a smattering of work focused on constraint satisfaction,
but in the main, there is no work that makes the satisfaction of hard
equality constraints central to its mission. Addressing this lack
is the motivation behind the paper. 

The principal difficulty in imposing linear (or other) equality constraints
lies in the nonlinear nature of the predictor of the output. In a
standard feedforward network, the predictor is often a pointwise nonlinearity
applied to a linear transformation of the output of the penultimate
layer. Enforcing linear constraints on the predictor using standard
Lagrangian methods is not straightforward as they effectively result
in a set of nonlinear constraints on the top layer weights---difficult
therefore (at first glance) to cast using standard convex optimization
methods. Not taking advantage of convex optimization and duality would
be unfortunate given the richness (and simplicity) of the methods
at hand. On the other hand, the softmax nonlinearity is a clear cut
example of a hard (sum to one) constraint being enforced, usually
at the top layer to obtain one-of-$K$ multi-class soft memberships.
This seems to run counter to our argument that enforcing constraints
on nonlinear predictors is difficult. We therefore ask (and answer
in the affirmative) the following question: can this example be generalized
to other constraints?

We show that linear equality constraints can be expressed in feedforward
neural networks using standard convex optimization methods. To achieve
this, we introduce auxiliary variables (one set for each training
set instance) and express the constraints on them. It turns out that
we obtain a \emph{maximization} problem on the auxiliary variables
(which take the form of predictors) and even more crucially, we obtain
a convex dual \emph{minimization} problem on the Lagrange parameter
vectors upon elimination of the auxiliary set. Furthermore, if the
form of the output nonlinearity is carefully matched with the choice
of the output divergence measure, the Lagrange parameters can be interpreted
as additional penultimate layer hidden units which connect to the
output via fixed weights (depending on the constraints). This is a
crucial discovery since these additional variables can be coupled
by standard means to lower layers (finding application in compression
autoencoders for example). No saddle-point methods are needed and
while it may be necessary to explicitly solve for the Lagrange parameters
in certain applications, we can deploy standard stochastic gradient
descent methods here---a huge computational advantage. 

The development of the basic idea proceeds as follows. We begin with
the multi-class logistic regression (MC-LR) negative log-likelihood
(which can serve as the top layer) and note (as mentioned above) that
the softmax nonlinearity of the predictor satisfies a sum to one constraint.
Then, we use an Ansatz (an educated guess) for the form of the saddle-point
Lagrangian with auxiliary variables, which when eliminated results
in a dual minimization on the Lagrange parameter used to set up the
softmax one-of-$K$ membership constraint. Subsequently, we show that
a difference between squared $\ell_{2}$ norms can serve as a principle
for setting up a saddle-point Lagrangian with linear constraints---albeit
one with a linear output layer. Moving forward to discover a principle,
the next step in the technical development is crucial: we identify
the difference of squared $\ell_{2}$ norms as a special case of the
difference of \emph{Bregman} \emph{divergences }with a direct relationship
between the choice of the output nonlinearity and the form of the
Bregman divergence \cite{Bregman1967200}. Once this identification
has been made, the remainder of the development is entirely straightforward.
Auxiliary variables (arranged in the form of predictors) are eliminated
from the difference of Bregman divergences saddle-point Lagrangian
leaving behind a dual \emph{minimization} w.r.t. the Lagrange parameter
vectors while continuing to obtain a standard minimization problem
w.r.t. the network weights. Matching the nonlinearity to the choice
of the Bregman divergence is important and we further require the
output nonlinearity to be continuous, differentiable and monotonic.
Provided these conditions are met, we obtain penultimate layer hidden
unit semantics for the Lagrange parameters---a surprising (but obvious
in hindsight) result. 

\section{A summary of the constraint satisfaction approach}\label{sec:constraint_satisfaction}

We present a very brief summary of the basic result. Assume a feedforward
neural network wherein we have the situation that the targets $\left\{ \boldsymbol{y}_{n}\right\} _{n=1}^{N}$
with $\boldsymbol{y}_{n}\in\mathbb{R}^{K}$ satisfy the constraints
$A\boldsymbol{y}_{n}=\boldsymbol{b}_{n}$ where $\boldsymbol{b}_{n}\in\mathbb{R}^{I}$.
Here the matrix of constraint weights $A\equiv\left[\boldsymbol{a}_{1},\ldots,\boldsymbol{a}_{K}\right]$.
We would therefore like the nonlinear predictor from the top layer
of the network also satisfy (or approximately satisfy) these constraints.
Assume top layer weights $W\equiv\left[\boldsymbol{w}_{1},\ldots,\boldsymbol{w}_{K}\right]$
and \emph{penultimate} layer hidden units $\boldsymbol{x}\in\mathbb{R}^{J}$.
(It should be clear that $K>I$ and that usually $J>K$.) Then, consider
the following loss function of the neural network (written in terms
of the top layer weights and for penultimate layer hidden units and
for a single instance):

\begin{align}
\ell(W,\boldsymbol{\lambda}) & =-\boldsymbol{y}^{T}\left(W^{T}\boldsymbol{x}+A^{T}\boldsymbol{\lambda}\right)+\boldsymbol{e}^{T}\boldsymbol{\phi}(W^{T}\boldsymbol{x}+A^{T}\boldsymbol{\lambda})\label{eq:loss_func}\\
 & =\sum_{k=1}^{K}\left[-y_{k}\left(\boldsymbol{w}_{k}^{T}\boldsymbol{x}+\boldsymbol{a}_{k}^{T}\boldsymbol{\lambda}\right)+\phi\left(\boldsymbol{w}_{k}^{T}\boldsymbol{x}+\boldsymbol{a}_{k}^{T}\boldsymbol{\lambda}\right)\right]\nonumber 
\end{align}
where $\boldsymbol{e}$ is the vector of all ones, $\boldsymbol{\lambda}\in\mathbb{R}^{I}$
and $\boldsymbol{\phi}$ is a vector of convex, differentiable functions
$\boldsymbol{\phi}(\boldsymbol{u})\equiv\left[\phi(u_{1}),\ldots,\phi(u_{K})\right]^{T}$.
We assume the output nonlinearity corresponding to the nonlinear predictor
is pointwise and chosen to be $z=\sigma(u)\equiv\phi^{\prime}(u)$.

We now minimize the loss function w.r.t. $\boldsymbol{\lambda}$ in
(\ref{eq:loss_func}) by differentiating and setting the result to
the origin to get

\begin{equation}
\nabla_{\boldsymbol{\lambda}}\ell(W,\boldsymbol{\lambda})=-A\boldsymbol{y}+A\boldsymbol{z}=\boldsymbol{0}\Rightarrow-\boldsymbol{b}+A\boldsymbol{z}=\boldsymbol{0}
\end{equation}
where $\boldsymbol{z}=\left[\phi^{\prime}(\boldsymbol{w}_{1}^{T}\boldsymbol{x}+\boldsymbol{a}_{1}^{T}\boldsymbol{\lambda}),\ldots,\phi^{\prime}(\boldsymbol{w}_{K}^{T}\boldsymbol{x}+\boldsymbol{a}_{K}^{T}\boldsymbol{\lambda})\right]^{T}$.
This implies that the loss function in (\ref{eq:loss_func}) is an
objective function w.r.t. $\boldsymbol{\lambda}$ as well as the weights
of the feedforward neural network. Furthermore, there exists one $\boldsymbol{\lambda}$
vector per input instance and when we minimize the loss function w.r.t.
$\boldsymbol{\lambda}$, we obtain a nonlinear predictor \textbf{$\boldsymbol{z}$}
which satisfies the output constraints $A\boldsymbol{z}=\boldsymbol{b}$.
The price of admission therefore is an additional set of \textbf{$\boldsymbol{\lambda}$}
units (with fixed weights $A$) which need to be updated in each epoch.
In the rest of the paper, we develop the theory underlying the above
loss function.

\section{Formal development of saddle-point Lagrangians and their dual }\label{sec:Diff_Breg_Lag}

\subsection{The Multi-class Logistic Regression (MC-LR) Lagrangian}

The negative log-likelihood for multi-class logistic regression is

\begin{equation}
-\log\Pr(\boldsymbol{Y}=\boldsymbol{y}|W,\boldsymbol{x})=-\sum_{k=1}^{K}y_{k}\boldsymbol{w}_{k}^{T}\boldsymbol{x}+\log\sum_{k=1}^{K}\exp\left\{ \boldsymbol{w}_{k}^{T}\boldsymbol{x}\right\} \label{eq:MC-LR_NLL}
\end{equation}
where the weight matrix $W$ comprises a set of weight vectors $W=\left[\boldsymbol{w}_{1},\ldots,\boldsymbol{w}_{K}\right]$,
the class label column vector $\boldsymbol{y}=\left[y_{1},\ldots,y_{K}\right]^{T}$
and the log-likelihood is expressed for a generic input-output training
set pair $\left\{ \boldsymbol{x},\boldsymbol{y}\right\} $. 

Rather than immediately embark on a tedious formal exposition of the
basic idea, we try and provide intuition. We begin with an Ansatz
(an informed guess based on a result shown in \cite{yuille1994statistical})
that leads to (\ref{eq:MC-LR_NLL}) and then provide a justification
for that educated guess. Consider the following Lagrangian Ansatz
for multi-class logistic regression: 
\begin{align}
\mathcal{L}_{\mathrm{MC-LR}}^{(\mathrm{total})}(W,\left\{ \boldsymbol{z}_{n}\right\} ,\left\{ \lambda_{n}\right\} ) & =\sum_{n=1}^{N}\left[\left(\boldsymbol{z}_{n}-\boldsymbol{y}_{n}\right)^{T}W^{T}\boldsymbol{x}_{n}-\boldsymbol{e}^{T}\boldsymbol{\psi}(\boldsymbol{z}_{n})+\lambda_{n}\left(\boldsymbol{e}^{T}\boldsymbol{z}_{n}-1\right)\right]\nonumber \\
 & =\sum_{n=1}^{N}\Big[\left(z_{nk}-y_{nk}\right)\sum_{j}w_{jk}x_{nj}-\sum_{k}\psi(z_{nk})+\lambda_{n}\Big(\sum_{k}z_{nk}-1\Big)\Big]\label{eq:MC-LR_Lagrangian}
\end{align}
where $\boldsymbol{z}=\left[z_{1},\ldots,z_{K}\right]^{T}$ is the
auxiliary variable \emph{predictor} of the output, $\boldsymbol{\psi}(\boldsymbol{z})\equiv\left[\psi(z_{1}),\ldots,\psi(z_{K})\right]^{T}$
is the \emph{negative entropy} and $\boldsymbol{e}$ is the column
vector of all ones with

\begin{equation}
\psi(z)=z\log z-z.\label{eq:phi_zlogz}
\end{equation}
(We use the term negative entropy for $\psi$ since it returns the
negative of the Shannon entropy for the standard sigmoid.) The one-of-$K$
encoding constraint is expressed by a Lagrange parameter $\lambda$
(one per training set instance). In the context of the rest of the
paper, the specific nonlinearity (in place of the sigmoid nonlinearity)
used for MC-LR is $\sigma(u)\equiv\exp\left\{ u\right\} $, its indefinite
integral---the sigmoid integral---$\phi(u)\equiv\exp\left\{ u\right\} $,
the inverse sigmoid $u(z)\equiv\sigma^{-1}(z)=\left(\phi^{\prime}\right)^{-1}(z)=\log z$
and the negative entropy $\psi(z)\equiv zu(z)-\phi(u(z))=z\log z-z$
(effectively a Legendre transform of $\phi(u)$ \cite{legendre1787memoire,mjolsness1990algebraic}).
The notation used in (\ref{eq:MC-LR_Lagrangian}) and in the rest
of the paper is now described. 

\subsubsection*{Notation:}\label{subsec:Notation:}

\noindent$\sigma(u),\,$$\phi(u),\,$$u(z)$ and $\psi(z)$ are the
generalized sigmoid, sigmoid integral, inverse sigmoid and negative
entropy respectively and correspond to individual variables ($u$
and $z$ both in $\mathbb{R}).$ Please see Table~\ref{tab:Sigmoid-to-Bregman}
for examples. Boldfaced quantities always correspond to vectors and/or
vector functions as the case may be. For example, $\boldsymbol{\sigma}(\boldsymbol{u})$
is not a tensor and is instead unpacked as a vector of functions $\left[\sigma(u_{1}),\ldots,\sigma(u_{K})\right]^{T}$.
The same applies for $\boldsymbol{u}(\boldsymbol{z}),$~$\boldsymbol{\phi}(\boldsymbol{u}(\boldsymbol{z}))$
and $\boldsymbol{\psi}(\boldsymbol{z})$. Likewise, $\nabla\boldsymbol{\psi}(\boldsymbol{v})$
is not a gradient tensor but refers to the vector of derivatives $\left[\psi^{\prime}(v_{1}),\ldots,\psi^{\prime}(v_{K})\right]^{T}$.
$\boldsymbol{e}$ is the vector of all ones with implied dimensionality.
Similarly, $\nabla\boldsymbol{u}(\boldsymbol{z})$ is a vector of
derivatives $\left[u^{\prime}(z_{1}),\ldots,u^{\prime}(z_{K})\right]^{T}$
and $\boldsymbol{z}\odot\nabla\boldsymbol{u}(\boldsymbol{z})$ is
a vector with $\odot$ denoting component-wise multiplication, i.e.
$\left[z_{1}u^{\prime}(z_{1}),\ldots,z_{K}u^{\prime}(z_{K})\right]^{T}$.
In an identical fashion, $\nabla_{\boldsymbol{u}}\boldsymbol{\phi}(\boldsymbol{u}(\boldsymbol{z}))$
is a vector of derivatives $\left[\phi^{\prime}(u(z_{1})),\ldots,\phi^{\prime}(u(z_{K}))\right]^{T}$
and $\nabla\boldsymbol{\phi}(\boldsymbol{u}(\boldsymbol{z}))\odot\nabla\boldsymbol{u}(\boldsymbol{z})$
denotes the vector formed by component-wise multiplication $\left[\phi^{\prime}(u(z_{1}))u^{\prime}(z_{1}),\ldots,\phi^{\prime}(u(z_{K}))u^{\prime}(z_{K})\right]^{T}$.
The Bregman divergence $B(\boldsymbol{z}||\boldsymbol{v})\equiv\sum_{k=1}^{K}B(z_{k}||v_{k})$---the
sum of component-wise Bregman divergences. Uppercase variables (such
as $W,\,A$) denote matrices with entries $\left\{ w_{kj}\right\} $
and $\{a_{ki}\}$ respectively and $O$ denotes the origin. Lagrangian
notation $\mathcal{L}$ is strictly reserved for (optimization or
saddle-point) problems \emph{with constraints}. Individual training
set instance pairs $(\boldsymbol{x},\boldsymbol{y})$ are often referred
to with the instance subscript index $n$ suppressed. 
\begin{center}
\begin{table}
\centering
\begin{tabular}{|c|c|c|c|c|c|}
\hline 
\multicolumn{2}{|c|}{Generalized sigmoid $\sigma(u)$} & \multirow{2}{*}{\begin{cellvarwidth}[t]
\centering
Sigmoid integral\\
 $\phi(u)$
\end{cellvarwidth}} & \multirow{2}{*}{\begin{cellvarwidth}[t]
\centering
Inverse\\
sigmoid $u(z)$
\end{cellvarwidth}} & \multirow{2}{*}{\begin{cellvarwidth}[t]
\centering
Negative entropy \\
$\psi(z)$
\end{cellvarwidth}} & \multirow{2}{*}{\begin{cellvarwidth}[t]
\centering
Bregman divergence \\
$B(z||v)$
\end{cellvarwidth}}\tabularnewline
\cline{1-2}
Name & Expression &  &  &  & \tabularnewline
\hline 
exp & $\exp\left\{ u\right\} $ & $\exp\left\{ u\right\} $ & $\log z$ & $z\log z-z$ & $z\log\frac{z}{v}-z+v$\tabularnewline
\hline 
\multirow{1}{*}{sig} & \multirow{1}{*}{$\frac{1}{1+\exp\left\{ -u\right\} }$} & \multirow{1}{*}{$\log\left(1+\exp\left\{ u\right\} \right)$} & \multirow{1}{*}{$\log\frac{z}{1-z}$} & \begin{cellvarwidth}[t]
\centering
$z\log z+$\\
$(1-z)\log(1-z)$
\end{cellvarwidth} & \begin{cellvarwidth}[t]
\centering
$z\log\frac{z}{v}+$\\
$(1-z)\log\frac{1-z}{1-v}$
\end{cellvarwidth}\tabularnewline
\hline 
tanh  & $\tanh u$ & $\log\cosh u$ & $\frac{1}{2}\log\frac{1+z}{1-z}$ & \begin{cellvarwidth}[t]
\centering
$\frac{1+z}{2}\log(1+z)+$\\
$\frac{1-z}{2}\log(1-z)$
\end{cellvarwidth} & \begin{cellvarwidth}[t]
\centering
$\frac{1+z}{2}\log\frac{1+z}{1+v}+$\\
$\frac{1-z}{2}\log\frac{1-z}{1-v}$
\end{cellvarwidth}\tabularnewline
\hline 
SRLU  & $\frac{au\exp\left\{ \frac{au^{2}}{2}\right\} +bu\exp\left\{ \frac{bu^{2}}{2}\right\} }{\exp\left\{ \frac{au^{2}}{2}\right\} +\exp\left\{ \frac{bu^{2}}{2}\right\} }$ & \begin{cellvarwidth}[t]
\centering
$\log\Big(\exp\left\{ \frac{au^{2}}{2}\right\} $\\
$+\exp\left\{ \frac{bu^{2}}{2}\right\} \Big)$
\end{cellvarwidth} & -\textsuperscript{1} & - & -\tabularnewline
\hline 
$\ell_{2}$ & $u$ & $\frac{u^{2}}{2}$ & $z$ & $\frac{z^{2}}{2}$ & $\frac{1}{2}\left(z-v\right)^{2}$\tabularnewline
\hline 
\end{tabular}\caption{Generalized sigmoid to Bregman divergence table. SRLU stands for soft
rectified linear unit. (\protect\textsuperscript{1}$u(z)$ for the
choice $\sigma(u)=\frac{au\exp\left\{ \frac{au^{2}}{2}\right\} +bu\exp\left\{ \frac{bu^{2}}{2}\right\} }{\exp\left\{ \frac{au^{2}}{2}\right\} +\exp\left\{ \frac{bu^{2}}{2}\right\} }$
is left implicit since the inverse is not available in closed form.
The same applies for the negative entropy and the Bregman divergence.)}\label{tab:Sigmoid-to-Bregman}
\end{table}
\par\end{center}

The Lagrangian in (\ref{eq:MC-LR_Lagrangian}) expresses a \emph{saddle-point}
problem---a minimization problem on $W$ and a maximization problem
on the predictors $\left\{ \boldsymbol{z}_{n}\right\} $. We will
show that the Lagrangian is concave w.r.t. $\left\{ \boldsymbol{z}_{n}\right\} $
leading to a convex dual minimization on $\left\{ \lambda_{n}\right\} $.
It turns out to be more convenient to work with a single instance
Lagrangian rather than (\ref{eq:MC-LR_Lagrangian}) which is on the
entire training set. This is written as

\begin{equation}
\mathcal{L}_{\mathrm{MC-LR}}(W,\boldsymbol{z},\lambda)=\left(\boldsymbol{z}-\boldsymbol{y}\right)^{T}W^{T}\boldsymbol{x}-\boldsymbol{e}^{T}\boldsymbol{\psi}(\boldsymbol{z})+\lambda\left(\boldsymbol{e}^{T}\boldsymbol{z}-1\right).\label{eq:MC-LR_Lagrangian_single}
\end{equation}

\begin{lem}
\label{lem:Lag_MC-LR_zlamsol}Eliminating $\boldsymbol{z}$ via maximization
in (\ref{eq:MC-LR_Lagrangian_single}) and solving for the Lagrange
parameter $\lambda$ via minimization yield the multi-class logistic
regression negative log-likelihood in (\ref{eq:MC-LR_NLL}).
\end{lem}
\begin{proof}
Since $\boldsymbol{\psi}(\boldsymbol{z})$ is continuous and differentiable,
we differentiate and solve for $\boldsymbol{z}$ in (\ref{eq:MC-LR_Lagrangian_single}).
Note that (\ref{eq:MC-LR_Lagrangian_single}) contains a sum of linear
and concave (since $\psi(z)=z\log z-z$ is convex) terms on $\boldsymbol{z}$
and is therefore concave. The negative of the Lagrangian is convex
and differentiable w.r.t. \textbf{$\boldsymbol{z}$} and with a linear
constraint on $\boldsymbol{z}$, we are guaranteed that the dual is
convex (and without any duality gap) \cite{boyd2004convex,slater2014lagrange}.
We obtain
\begin{equation}
\log z_{k}=\boldsymbol{w}_{k}^{T}\boldsymbol{x}+\lambda\Rightarrow z_{k}=\exp\left\{ \boldsymbol{w}_{k}^{T}\boldsymbol{x}+\lambda\right\} .
\end{equation}
(Note that the predictor $z_{k}$ has a nonlinear dependence on both
$W$ and $\lambda$ through the choice of sigmoid: $\exp$.) We eliminate
$\boldsymbol{z}$ to obtain the dual w.r.t. $\lambda$ written as

\begin{align}
R_{\mathrm{MC-LR}}(W,\lambda) & =-\sum_{k=1}^{K}y_{k}\boldsymbol{w}_{k}^{T}\boldsymbol{x}-\lambda+\sum_{k=1}^{K}\exp\left\{ \boldsymbol{w}_{k}^{T}\boldsymbol{x}+\lambda\right\} \nonumber \\
 & =-\sum_{k=1}^{K}y_{k}\boldsymbol{w}_{k}^{T}\boldsymbol{x}-\lambda+\sum_{k=1}^{K}\phi(\boldsymbol{w}_{k}^{T}\boldsymbol{x}+\lambda)\nonumber \\
 & =-\sum_{k=1}^{K}y_{k}\left(\boldsymbol{w}_{k}^{T}\boldsymbol{x}+\lambda\right)+\sum_{k=1}^{K}\phi(\boldsymbol{w}_{k}^{T}\boldsymbol{x}+\lambda)\label{eq:MC-LR_lambda}
\end{align}
where we have used the fact that $\sum_{k=1}^{K}y_{k}=1$. Also $\phi(u)=\exp\left\{ u\right\} $
is the sigmoid integral corresponding to the choice of sigmoid $z(u)=\exp\left\{ u\right\} $.
Please see section~\ref{subsec:Notation:} and Table~\ref{tab:Sigmoid-to-Bregman}
for more information on the notation used. It should be understood
that we continue to have a minimization problem on $W$ but now augmented
with a convex minimization problem on $\lambda$. Differentiating
and solving for $\lambda$ in (\ref{eq:MC-LR_lambda}), we get

\begin{equation}
-1+\sum_{k=1}^{K}\exp\left\{ \boldsymbol{w}_{k}^{T}\boldsymbol{x}+\lambda\right\} =0\Rightarrow\lambda=-\log\sum_{k=1}^{K}\exp\left\{ \boldsymbol{w}_{k}^{T}\boldsymbol{x}\right\} .
\end{equation}
Substituting this solution for $\lambda$ in (\ref{eq:MC-LR_lambda}),
we get

\begin{equation}
R_{\mathrm{MC-LR}}(W,\lambda(W))=-\sum_{k=1}^{K}y_{k}\boldsymbol{w}_{k}^{T}\boldsymbol{x}+\log\sum_{k=1}^{K}\exp\left\{ \boldsymbol{w}_{k}^{T}\boldsymbol{x}\right\} +\mathrm{constant}
\end{equation}
which is identical (up to additive constant terms) to (\ref{eq:MC-LR_NLL}). 
\end{proof}
We can write the dual (with the optimization on $W$) as a new MC-LR
objective function:

\begin{equation}
E_{\mathrm{MC-LR}}^{(\mathrm{total})}(W,\left\{ \lambda_{n}\right\} )=\sum_{n=1}^{N}\left[-\boldsymbol{y}_{n}^{T}W^{T}\boldsymbol{x}_{n}-\lambda_{n}+\boldsymbol{e}^{T}\boldsymbol{\phi}\left(W^{T}\boldsymbol{x}_{n}+\lambda_{n}\boldsymbol{e}\right)\right].\label{eq:E_MC-LR_total}
\end{equation}
It should be clear from Lemma~\ref{lem:Lag_MC-LR_zlamsol} (that
this is a \emph{minimization} problem on $(W,\left\{ \lambda_{n}\right\} )$
despite the presence of the set of Lagrange parameters $\left\{ \lambda_{n}\right\} $.
The structure of the objective function strongly indicates that the
Lagrange parameter $\lambda_{n}$ is almost on the same footing as
the set of penultimate layer hidden units $\boldsymbol{x}_{n}$: it
appears inside the nonlinearity in the form of a new unit with a fixed
weight of one and also connects to the loss in the same manner. Noting
that $\boldsymbol{e}^{T}\boldsymbol{y}_{n}=1,$ we rewrite (\ref{eq:E_MC-LR_total})
as

\begin{equation}
E_{\mathrm{MC-LR}}^{(\mathrm{total})}(W,\left\{ \lambda_{n}\right\} )=\sum_{n=1}^{N}\left[-\boldsymbol{y}_{n}^{T}\left(W^{T}\boldsymbol{x}_{n}+\lambda_{n}\boldsymbol{e}\right)+\boldsymbol{e}^{T}\boldsymbol{\phi}\left(W^{T}\boldsymbol{x}_{n}+\lambda_{n}\boldsymbol{e}\right)\right]
\end{equation}
which is similar to (\ref{eq:loss_func}) but with the specialization
to multi-class classification. This simple observation serves as an
\emph{underpinning} for the entire paper. 

\subsection{Difference of squared $\ell_{2}$ norms Lagrangian}

We now embark on clarifying the intuition that led to the Ansatz in
(\ref{eq:MC-LR_Lagrangian}). Since divergence measures are closely
coupled to nonlinearities in neural networks, it can be quite difficult
to guess at the right Lagrangian which allows us to move from a saddle-point
problem to a dual minimization (on weight matrices and Lagrange parameters).
To (initially) remove nonlinearities completely from the picture,
consider the following saddle-point Lagrangian with a minimization
problem on $W$, a maximization problem on $\boldsymbol{z}$ and the
constraint $A\boldsymbol{z}=\boldsymbol{b}$ for a single training
set pair $(\boldsymbol{x},\boldsymbol{y}).$ 

\begin{equation}
\mathcal{L}_{\ell_{2}}(W,\boldsymbol{z},\boldsymbol{\lambda})=\frac{1}{2}\|\boldsymbol{y}-W^{T}\boldsymbol{x}\|_{2}^{2}-\frac{1}{2}\|\boldsymbol{z}-W^{T}\boldsymbol{x}\|_{2}^{2}+\boldsymbol{\lambda}^{T}(A\boldsymbol{z}-\boldsymbol{b}).\label{eq:L2_Lagrangian_single}
\end{equation}
Without the constraints, the predictor attains its maximum for $\boldsymbol{z}=W^{T}\boldsymbol{x}$,
the optimal solution and we recover the original unconstrained learning
problem on the weights. With the constraints, the maximization problem
is not (usually) allowed to reach its unconstrained value with the
predictor likewise not allowed to reach $W^{T}\boldsymbol{x}$ due
to constraint satisfaction. With the constrained maximization problem
nested inside the minimization problem, the intuition here is that
the maximization problem leads to a predictor $\boldsymbol{z}$ which
depends only on the input $\boldsymbol{x}$ (and not directly on $\boldsymbol{y}$)
and satisfies the constraints $(A\boldsymbol{z}=\boldsymbol{b})$.
Furthermore, elimination of $\boldsymbol{z}$ leads to a dual minimization
problem on $\boldsymbol{\lambda}$ which can be combined with the
minimization problem on $W$ in exactly the same way as in MC-LR above.
After this elimination step, $W$ is tuned to best fit to the ``labels''
$\boldsymbol{y}$ (but with a regularization term obtained from the
maximization problem, one that penalizes deviation from the constraints).
By this mechanism, the weight gradient of the objective function gets
projected into the subspace spanned by the constraints. Consequently,
once initial conditions are set up to satisfy the constraints, the
network cannot subsequently deviate from them. Since each training
set instance $(\boldsymbol{x},\boldsymbol{y})$ is deemed to satisfy
$A\boldsymbol{y}=\boldsymbol{b}$, we can separately evaluate the
gradient term for each training set pair for constraint satisfaction. 
\begin{lem}
If $A\boldsymbol{y}=\boldsymbol{b}$, then A$\nabla_{W}\mathcal{\mathcal{L}}_{\ell_{2}}(W,\boldsymbol{z}(W),\boldsymbol{\lambda}(W))=O,$
or in other words the gradient w.r.t. $W$ of the Lagrangian in (\ref{eq:L2_Lagrangian_single})
with $\boldsymbol{z}$ and $\boldsymbol{\lambda}$ set to their optimal
values is in the nullspace of $A$.
\end{lem}
\begin{proof}
Since the Lagrangian in (\ref{eq:L2_Lagrangian_single}) is quadratic
w.r.t. $\boldsymbol{z}$, we differentiate and solve to get

\begin{equation}
\boldsymbol{z}=W^{T}\boldsymbol{x}+A^{T}\boldsymbol{\lambda}.\label{eq:zl2sol}
\end{equation}
Substituting this solution for $\boldsymbol{z}$ in (\ref{eq:L2_Lagrangian_single}),
we obtain the dual objective function (while continuing to have a
minimization problem on $W$): 

\begin{align}
R_{\ell_{2}}(W,\boldsymbol{\lambda}) & =\frac{1}{2}\|\boldsymbol{y}-W^{T}\boldsymbol{x}\|_{2}^{2}+\frac{1}{2}\boldsymbol{\lambda}^{T}AA^{T}\boldsymbol{\lambda}+\boldsymbol{\lambda}^{T}AW^{T}\boldsymbol{x}-\boldsymbol{\lambda}^{T}\boldsymbol{b}\nonumber \\
 & \propto-\boldsymbol{y}^{T}W^{T}\boldsymbol{x}-\boldsymbol{b}^{T}\boldsymbol{\lambda}+\frac{1}{2}\boldsymbol{x}^{T}WW^{T}\boldsymbol{x}+\boldsymbol{x}^{T}WA^{T}\boldsymbol{\lambda}+\frac{1}{2}\boldsymbol{\lambda}^{T}AA^{T}\boldsymbol{\lambda}\nonumber \\
 & =-\boldsymbol{y}^{T}W^{T}\boldsymbol{x}-\boldsymbol{b}^{T}\boldsymbol{\lambda}+\frac{1}{2}\|W^{T}\boldsymbol{x}+A^{T}\boldsymbol{\lambda}\|_{2}^{2}.\label{eq:RL2_lambda}
\end{align}
This objective function is similar to the previous one obtained for
MC-LR in (\ref{eq:MC-LR_lambda}) and to help see the intuition, we
rewrite it as

\begin{align}
S_{\ell_{2}}(W,\boldsymbol{z},\boldsymbol{\lambda}) & =-\boldsymbol{y}^{T}W^{T}\boldsymbol{x}-\boldsymbol{y}^{T}A^{T}\boldsymbol{\lambda}+\boldsymbol{y}^{T}A^{T}\boldsymbol{\lambda}-\boldsymbol{b}^{T}\boldsymbol{\lambda}+\frac{1}{2}\|W^{T}\boldsymbol{x}+A^{T}\boldsymbol{\lambda}\|_{2}^{2}\nonumber \\
 & =-\boldsymbol{y}^{T}\boldsymbol{z}+\frac{1}{2}\boldsymbol{z}^{T}\boldsymbol{z}+\boldsymbol{\lambda}^{T}(A\boldsymbol{y}-\boldsymbol{b})\,\,\,\,\,\,\,\,\,\,\,\,(\mathrm{since}\,\boldsymbol{z}=W^{T}\boldsymbol{x}+A^{T}\boldsymbol{\lambda})\nonumber \\
 & =-\boldsymbol{y}^{T}\boldsymbol{z}+\frac{1}{2}\boldsymbol{z}^{T}\boldsymbol{z}
\end{align}
since $A\boldsymbol{y}=\boldsymbol{b}$ and with $\boldsymbol{z}$
as in (\ref{eq:zl2sol}). This clearly shows the constrained predictor
$\boldsymbol{z}=W^{T}\boldsymbol{x}+A^{T}\boldsymbol{\lambda}$ attempting
to fit to the labels $\boldsymbol{y}$. 

We now enforce the constraint $A\boldsymbol{z}=\boldsymbol{b}$ in
(\ref{eq:zl2sol}) to get

\begin{equation}
AW^{T}\boldsymbol{x}+AA^{T}\boldsymbol{\lambda}=\boldsymbol{b}\Rightarrow\boldsymbol{\lambda}=\left(AA^{T}\right)^{-1}\left(\boldsymbol{b}-AW^{T}\boldsymbol{x}\right)\label{eq:RL2_solution_lambda}
\end{equation}
provided $(AA^{T})^{-1}$ exists. Using the solution in (\ref{eq:RL2_solution_lambda})
and eliminating $\boldsymbol{\lambda}$ from the objective function
in (\ref{eq:L2_Lagrangian_single}), we obtain 

\begin{equation}
R_{\ell_{2}}(W,\boldsymbol{\lambda}(W))\propto-\boldsymbol{y}^{T}W^{T}\boldsymbol{x}+\boldsymbol{x}^{T}WQ\boldsymbol{b}+\frac{1}{2}\boldsymbol{x}^{T}WPW^{T}\boldsymbol{x}\label{eq:RL2_final}
\end{equation}
where $P\equiv(I-A^{T}(AA^{T})^{-1}A)$ is a projection matrix (with
$P^{2}=P$) and $Q\equiv A^{T}(AA^{T})^{-1}$ and with properties
$PQ=O,$ $AP=O$ and $AQ=I$ which can be readily verified. The objective
function w.r.t. $W$ is clearly bounded from below by virtue of the
last term in $(\ref{eq:RL2_final}).$ Taking the gradient of $R_{\ell_{2}}$
w.r.t. $W$, we obtain

\begin{equation}
\nabla_{W}R_{\ell_{2}}(W,\boldsymbol{\lambda}(W))=\left[-\boldsymbol{y}+Q\boldsymbol{b}+PW^{T}\boldsymbol{x}\right]\boldsymbol{x}^{T}.
\end{equation}
From this, we get

\begin{equation}
A\nabla_{W}R_{\ell_{2}}(W,\boldsymbol{\lambda}(W))=\left[-A\boldsymbol{y}+AQ\boldsymbol{b}+APW^{T}\boldsymbol{x}\right]\boldsymbol{x}^{T}=\left(-A\boldsymbol{y}+\boldsymbol{b}\right)\boldsymbol{x}^{T}=O
\end{equation}
since $AQ=I,\,AP=O$ and provided the target $\boldsymbol{y}$ satisfies
$A\boldsymbol{y}=\boldsymbol{b}$ which was our starting assumption. 
\end{proof}
We have shown that the gradient of the effective objective function
w.r.t. $W$ for each training set pair $(\boldsymbol{x},\boldsymbol{y})$
obtained by eliminating $\boldsymbol{z}$ and \textbf{$\boldsymbol{\lambda}$}
from the Lagrangian in (\ref{eq:L2_Lagrangian_single}) lies in the
nullspace of $A$. This observation may be useful mainly for analysis
since actual constraint satisfaction may be computationally more efficient
in the dual objective function (written for all training set instances
for the sake of completion):

\begin{align}
E_{\mathrm{\ell_{2}}}^{(\mathrm{total})}(W,\left\{ \boldsymbol{\lambda}_{n}\right\} ) & =\sum_{n=1}^{N}\left[-\boldsymbol{y}_{n}^{T}W^{T}\boldsymbol{x}_{n}-\boldsymbol{b}_{n}^{T}\boldsymbol{\lambda}_{n}+\frac{1}{2}\|W^{T}\boldsymbol{x}_{n}+A^{T}\boldsymbol{\lambda}_{n}\|_{2}^{2}\right]\nonumber \\
 & =\sum_{n=1}^{N}\left[-\boldsymbol{y}_{n}^{T}\left(W^{T}\boldsymbol{x}_{n}+A^{T}\boldsymbol{\lambda}_{n}\right)+\frac{1}{2}\|W^{T}\boldsymbol{x}_{n}+A^{T}\boldsymbol{\lambda}_{n}\|_{2}^{2}\right]\,\,\,\,\,\,(\mathrm{since}\,A\boldsymbol{y}_{n}=\boldsymbol{b}_{n})\nonumber \\
 & =\sum_{n=1}^{N}\left[-\boldsymbol{y}_{n}^{T}\boldsymbol{z}_{n}(W,\lambda_{n})+\frac{1}{2}\|\boldsymbol{z}_{n}(W,\lambda_{n})\|_{2}^{2}\right]\label{eq:El2_total}
\end{align}
where the predictor 

\begin{equation}
\boldsymbol{z}_{n}(W,\boldsymbol{\lambda}_{n})\equiv W^{T}\boldsymbol{x}_{n}+A^{T}\boldsymbol{\lambda}_{n}
\end{equation}
as before in (\ref{eq:zl2sol}) but modified to express the dependence
on the parameters $W$ and $\boldsymbol{\lambda}_{n}$. The surprising
fact that we can augment the penultimate layer units and the final
layer weights in this manner has its origin in the difference between
the squared $\ell_{2}$ norms used back in (\ref{eq:L2_Lagrangian_single}).
Therefore, we have shown that the components of the Lagrange parameter
vector $\boldsymbol{\lambda}$ can be interpreted as additional, penultimate
layer hidden units with fixed weights obtained from $A$. In the remainder
of the paper, we generalize this idea to Bregman divergences allowing
us to prove the same result for any suitable generalized sigmoid (under
the presence of linear equality constraints).

\subsection{The difference of Bregman divergences Lagrangian}

As we move from the specific (MC-LR and $\ell_{2}$) to the general
(Bregman divergence), we define some terms which simplify the notation
and development.
\begin{defn}
The generalized sigmoid vector $\boldsymbol{\sigma}(\boldsymbol{u})$
is defined as $\boldsymbol{\sigma}(\boldsymbol{u})=\left[\sigma(u_{1}),\ldots,\sigma(u_{K})\right]^{T}$
where $\sigma(u)$ is continuous, differentiable and monotonic ($\sigma^{\prime}(u)>0,\,\forall u$). 
\end{defn}
An example (different from the standard sigmoid) is $\sigma(u)=\tanh(u).$
Its derivative $\sigma^{\prime}(u)=\mathrm{sech}^{2}(u)>0,\,\forall u\in\left(-\infty,\infty\right)$.
\begin{defn}
\label{def:gensigint}The generalized sigmoid integral is defined
as the vector $\boldsymbol{\phi}(\boldsymbol{u})=\left[\phi(u_{1}),\ldots,\phi(u_{K})\right]^{T}$
where $\phi(u)=\int_{-\infty}^{u}\sigma(u)du$. 
\end{defn}
Continuing the $\sigma(u)=\tanh(u)$ example, the sigmoid integral
for $\sigma(u)=\tanh(u)$ is $\phi(u)=\log\cosh(u)$.
\begin{defn}
The generalized inverse sigmoid is defined as a vector $\boldsymbol{u}(\boldsymbol{z})=\left[u(z_{1}),\ldots,u(z_{K})\right]^{T}$
where $u(z)=\sigma^{-1}(z)=\left(\phi^{\prime}\right)^{-1}(z)$.
\end{defn}
The inverse sigmoid for $\sigma(u)=\tanh(u)$ is $u(z)=\frac{1}{2}\log\frac{1+z}{1-z}$.
\begin{defn}
The generalized negative entropy $\boldsymbol{\psi}(\boldsymbol{z})$
is defined as $\boldsymbol{z}^{T}\boldsymbol{u}(\boldsymbol{z})-\boldsymbol{e}^{T}\boldsymbol{\phi}(\boldsymbol{u}(\boldsymbol{z}))$.
\end{defn}
For $\sigma(u)=\tanh(u)$, the negative entropy $\psi(z)=\frac{1}{2}\left[(1+z)\log(1+z)+(1-z)\log(1-z)\right]$
and $\phi(u(z))=\frac{1}{2}\log(1-z^{2})$. The generalized negative
entropy $\psi(z)$ is related to the Legendre transform of $\phi(u)$
(and can be derived from Legendre-Fenchel duality theory \cite{fenchel2014conjugate,mjolsness1990algebraic})
but this relationship is not emphasized in the paper.
\begin{cor}
\label{lem:psi_convex}The generalized negative entropy $\boldsymbol{\psi}(\boldsymbol{z})$
is convex.
\end{cor}
\begin{proof}
Since $\boldsymbol{\psi}(\boldsymbol{z})$ is a sum of individual
functions $\psi(z)$, we just need to show that $\psi(z)$ is convex.
The derivative of $\psi(z)$ is

\begin{equation}
\psi^{\prime}(z)=u(z)+zu^{\prime}(z)-\phi^{\prime}(u(z))u^{\prime}(z)=u(z)+zu^{\prime}(z)-zu^{\prime}(z)=u(z)
\end{equation}
since $u(z)=\left(\phi^{\prime}\right)^{-1}(z)$. The second derivative
$\psi^{\prime\prime}(z)$ is

\begin{equation}
\psi^{\prime\prime}(z)=u^{\prime}(z)=\frac{1}{\sigma^{\prime}(u(z))}>0
\end{equation}
since $u(z)=\sigma^{-1}(z)$ and $\sigma^{\prime}(u)>0$. Therefore
the generalized negative entropy is convex.
\end{proof}
We can directly show the convexity of the negative entropy $\psi(z)$
for $\sigma(u)=\tanh(u)$ by evaluating $\psi^{\prime\prime}(z)=u^{\prime}(z)=\frac{d}{dz}\left(\frac{1}{2}\log\frac{1+z}{1-z}\right)=\frac{1}{1-z^{2}}>0,\forall z\in\left[-1,1\right]$.
We now set up the Bregman divergence \cite{Bregman1967200} for the
convex negative entropy function $\psi(z)$.
\begin{defn}
The Bregman divergence for a convex, differentiable function $\psi(z)$
is defined as 

\begin{equation}
B(z||v)=\psi(z)-\psi(v)-\psi^{\prime}(v)(z-v).\label{eq:Bregman}
\end{equation}
\end{defn}
This can be extended to the generalized negative entropy $\boldsymbol{\psi}(\boldsymbol{z})$
but the notation becomes a bit more complex. For the sake of completeness,
the Bregman divergence for $\boldsymbol{\psi}(\boldsymbol{z})$ is
written as

\begin{equation}
B(\boldsymbol{z}||\boldsymbol{v})=\boldsymbol{e}^{T}\left(\boldsymbol{\psi}(\boldsymbol{z})-\boldsymbol{\psi}(\boldsymbol{v})\right)-\left(\nabla\boldsymbol{\psi}(\boldsymbol{v})\right)^{T}(\boldsymbol{z}-\boldsymbol{v})
\end{equation}
where $\nabla\boldsymbol{\psi}(\boldsymbol{v})$ is the gradient vector
of the generalized negative entropy defined as $\nabla\boldsymbol{\psi}(\boldsymbol{v})\equiv\left[\psi^{\prime}(v_{1}),\ldots,\psi^{\prime}(v_{K})\right]^{T}$.

This definition can be used to build the Bregman divergence for the
generalized negative entropy $\psi(z)$ written as

\begin{align}
B(z||v) & =\psi(z)-\psi(v)-\psi^{\prime}(v)(z-v)\nonumber \\
 & =zu(z)-\phi(u(z))-vu(v)+\phi(u(v))-u(v)(z-v)\nonumber \\
 & =z\left[u(z)-u(v)\right]-\phi(u(z))+\phi(u(v)).
\end{align}
Therefore, the Bregman divergence for the choice of $\sigma(u)=\tanh(u)$
is 

\begin{align}
B_{\mathrm{tanh}}(z||v) & =z\left[\frac{1}{2}\log\frac{1+z}{1-z}-\frac{1}{2}\log\frac{1+v}{1-v}\right]-\frac{1}{2}\log\left(1-z^{2}\right)+\frac{1}{2}\log\left(1-v^{2}\right)\nonumber \\
 & =\frac{1}{2}\left[(1+z)\log\frac{1+z}{1+v}+(1-z)\log\frac{1-z}{1-v}\right]
\end{align}
which is the Kullback-Leibler (KL) divergence for a binary random
variable $Z$ taking values in $\left\{ -1,1\right\} $ which is what
we would expect once we make the identification that $z=2\Pr(Z=1)-1$
with a similar identification for $v$.
\begin{cor}
The Bregman divergence $B_{\mathrm{gensig}}(z||v)$ is convex in $z$
for the choice of the generalized negative entropy function $\psi(z)$.
\end{cor}
\begin{proof}
Since $B_{\mathrm{gensig}}(z||v)$ is a summation of $\psi(z)$ and
a term linear in $z$, it is convex in $z$ due to the convexity of
$\psi(z)$ (shown earlier in Lemma~\ref{lem:psi_convex}). 
\end{proof}
We now describe the crux of one of the basic ideas in this paper:
that the difference between squared $\ell_{2}$ norms in (\ref{eq:L2_Lagrangian_single})
when generalized to the difference of Bregman divergences leads to
a dual minimization problem wherein the Lagrange parameters are additional
penultimate layer hidden units. In (\ref{eq:L2_Lagrangian_single})
we showed that nesting a maximization problem on the constrained predictor
$\boldsymbol{z}$ inside the overall minimization problem w.r.t. $W$
had two consequences: (i) eliminating $\boldsymbol{z}$ led to a minimization
problem on $\boldsymbol{\lambda}$; (ii) the saddle-point problem
led to opposing forces---from the target $\boldsymbol{y}$ and from
the input $\boldsymbol{x}$---which led to projected gradients which
satisfied the constraints. Furthermore, the Lagrange parameters carried
the interpretation of penultimate layer hidden units (albeit in a
final layer with a linear output unit) via the form of the dual minimization
problem in (\ref{eq:El2_total}). We now generalize the difference
of squared $\ell_{2}$ norms to the difference of Bregman divergences
and show that the same interpretation holds for suitable nonlinearities
in the output layer. 

From the generalized negative entropy Bregman divergence, we form
the difference of Bregman divergences to get

\begin{align}
B_{\mathrm{gensig}}(y||v)-B_{\mathrm{gensig}}(z||v) & =(z-y)u(v)+yu(y)-\phi(u(y))-zu(z)+\phi(u(z)).\nonumber \\
 & =(z-y)u(v)+\psi(y)-\psi(z).\label{eq:Bregman_difference}
\end{align}

\begin{lem}
The difference of Bregman divergences is linear in the weight matrix
$W$ for the choice of the convex generalized negative entropy vector
$\boldsymbol{\psi}(\boldsymbol{z})$ and for $\boldsymbol{v}=\boldsymbol{\sigma}(W^{T}\boldsymbol{x})$.
\end{lem}
\begin{proof}
This is a straightforward consequence of the two choices.

\begin{align}
B_{\mathrm{gensig}}(y||v)-B_{\mathrm{gensig}}(z||v) & =y(u(y)-u(v))-\phi(u(y))+\phi(u(v))-\left[z(u(z)-u(v))-\phi(u(z))+\phi(u(v))\right]\nonumber \\
 & =(z-y)u(v)+yu(y)-\phi(u(y))-zu(z)+\phi(u(z))\nonumber \\
 & =(z-y)u(v)-zu(z)+\phi(u(z))+\mathrm{terms}\,\mathrm{independent}\,\mathrm{of}\,z\nonumber \\
 & =(z-y)u(v)-\psi(z)+\mathrm{terms}\,\mathrm{independent}\,\mathrm{of}\,z.\label{eq:Bregman_div_diff}
\end{align}
We write this for a single instance as

\begin{align}
B_{\mathrm{gensig}}(\boldsymbol{y}||\boldsymbol{v})-B_{\mathrm{gensig}}(\boldsymbol{z}||\boldsymbol{v}) & \propto(\boldsymbol{z}-\boldsymbol{y})^{T}\boldsymbol{u}(\boldsymbol{v})-\boldsymbol{z}^{T}\boldsymbol{u}(\boldsymbol{z})+\boldsymbol{e}^{T}\boldsymbol{\phi}(\boldsymbol{u}(\boldsymbol{z})).\\
 & =(\boldsymbol{z}-\boldsymbol{y})^{T}\boldsymbol{u}(\boldsymbol{v})-\boldsymbol{e}^{T}\boldsymbol{\psi}(\boldsymbol{z}).
\end{align}
Note that the difference of Bregman divergences depends on $\boldsymbol{v}$
only through $\boldsymbol{u}(\boldsymbol{v})$. We now set $\boldsymbol{v}=\boldsymbol{\sigma}(W^{T}\boldsymbol{x})$
and rewrite the difference between Bregman divergences using terms
that only depend on $W,\boldsymbol{x}$ and $\boldsymbol{z}$ to get

\begin{equation}
B_{\mathrm{gensig}}(\boldsymbol{y}||\boldsymbol{\sigma}(W^{T}\boldsymbol{x}))-B_{\mathrm{gensig}}(\boldsymbol{z}||\boldsymbol{\sigma}(W^{T}\boldsymbol{x}))\propto(\boldsymbol{z}-\boldsymbol{y})^{T}W^{T}\boldsymbol{x}-\boldsymbol{e}^{T}\boldsymbol{\psi}(\boldsymbol{z})\label{eq:Bregman_div_diff_WTx}
\end{equation}
which is linear in the weight matrix $W$. 
\end{proof}
\begin{cor}
The difference of Bregman divergences for the choice of $\phi(z)=\frac{z^{2}}{2}$
is identical to the difference of squared $\ell_{2}$ norms in (\ref{eq:L2_Lagrangian_single}).
\end{cor}
\begin{proof}
We evaluate the difference of Bregman divergences for $\phi(z)=\frac{z^{2}}{2}$
by substituting this choice of $\phi(z)$ in (\ref{eq:Bregman_difference})
to get 

\begin{align}
B_{\mathrm{\ell_{2}}}(y||v)-B_{\mathrm{\ell_{2}}}(z||v) & =(z-y)v+y^{2}-\frac{y^{2}}{2}-z^{2}+\frac{z^{2}}{2}\nonumber \\
 & =\frac{1}{2}(y-v)^{2}-\frac{1}{2}(z-v)^{2}
\end{align}
which when extended to the entire set of hidden units is the difference
between the squared $\ell_{2}$ norms as in (\ref{eq:L2_Lagrangian_single}).
\end{proof}
The general saddle-point problem can be now be expressed (for a single
training set pair) as

\begin{equation}
\min_{W}\max_{\boldsymbol{z}}B_{\mathrm{gensig}}(\boldsymbol{y}||\boldsymbol{\sigma}(W^{T}\boldsymbol{x}))-B_{\mathrm{gensig}}(\boldsymbol{z}||\boldsymbol{\sigma}(W^{T}\boldsymbol{x}))\propto\min_{W}\max_{\boldsymbol{z}}(\boldsymbol{z}-\boldsymbol{y})^{T}W^{T}\boldsymbol{x}-\boldsymbol{e}^{T}\boldsymbol{\psi}(\boldsymbol{z})\label{eq:minmaxBreg_diff}
\end{equation}
 
\[
\mathrm{subject}\,\mathrm{to}\,A\boldsymbol{z}=\boldsymbol{b}.
\]
The order of the maximization and minimization matters. It is precisely
because we have introduced a maximization problem on $\boldsymbol{z}$
\emph{inside} a minimization problem on the weights $W$ that we are
able to obtain a dual minimization problem on the Lagrange parameters
in exactly the same manner as in multi-class logistic regression.
This will have significant implications later. The difference of Bregman
divergences is \emph{concave} in $\boldsymbol{z}$. Note that the
weight matrix $W$ appears linearly in (\ref{eq:Bregman_div_diff_WTx})
due to the choice of $\boldsymbol{v}=\boldsymbol{\sigma}(W^{T}\boldsymbol{x})$
which produces the correct generalization of (\ref{eq:MC-LR_Lagrangian_single}).
We add the linear constraints $A\boldsymbol{z}=\boldsymbol{b}$ as
before and write the Lagrangian for a single instance

\begin{equation}
\mathcal{L}_{\mathrm{gensig}}(W,\boldsymbol{z},\boldsymbol{\lambda})=\left(\boldsymbol{z}-\boldsymbol{y}\right)^{T}W^{T}\boldsymbol{x}-\boldsymbol{e}^{T}\boldsymbol{\psi}(\boldsymbol{z})+\boldsymbol{\lambda}^{T}(A\boldsymbol{z}-\boldsymbol{b})\label{eq:Lagrangian_gensig_single}
\end{equation}
which can be expanded to the entire training set as 

\begin{equation}
\mathcal{L}_{\mathrm{gensig}}^{(\mathrm{total})}(W,\left\{ \boldsymbol{z}_{n}\right\} ,\left\{ \boldsymbol{\lambda}_{n}\right\} )=\sum_{n=1}^{N}\left[\left(\boldsymbol{z}_{n}-\boldsymbol{y}_{n}\right)^{T}W^{T}\boldsymbol{x}_{n}-\boldsymbol{e}^{T}\boldsymbol{\psi}(\boldsymbol{z}_{n})+\boldsymbol{\lambda}_{n}^{T}(A\boldsymbol{z}_{n}-\boldsymbol{b}_{n})\right].\label{eq:Lagrangian_gensig_total}
\end{equation}
where the negative entropy $\boldsymbol{\psi}(\boldsymbol{z})=\boldsymbol{z}^{T}\boldsymbol{u}(\boldsymbol{z})-\boldsymbol{e}^{T}\boldsymbol{\phi}(\boldsymbol{u}(\boldsymbol{z}))$.
The final saddle-point Lagrangian contains a minimization problem
on $W$ and a maximization problem on $\left\{ \boldsymbol{z}_{n}\right\} $
with the unknown Lagrange parameter vectors $\left\{ \boldsymbol{\lambda}_{n}\right\} $
handling the linear constraints. Note that $\boldsymbol{\phi}(\boldsymbol{u}(\boldsymbol{z}))$
lurking inside $\boldsymbol{\psi}(\boldsymbol{z})$ is a generalized
sigmoid integral (as per Definition~\ref{def:gensigint}) with $\boldsymbol{u}(\boldsymbol{z})$
being the \emph{matched} generalized inverse sigmoid. The generalized
sigmoid integral is not restricted to the exponential function. 

We now show that the generalized dual objective function on $\left\{ \boldsymbol{\lambda}_{n}\right\} $
(along with a minimization problem on $W$) is of the same form as
the MC-LR dual objective function (\ref{eq:E_MC-LR_total}) and is
written as 

\begin{equation}
E_{\mathrm{gensig}}^{(\mathrm{total})}(W,\left\{ \boldsymbol{\lambda}_{n}\right\} )=\sum_{n=1}^{N}\left[-\boldsymbol{y}_{n}^{T}\left(W^{T}\boldsymbol{x}_{n}+A^{T}\boldsymbol{\lambda}_{n}\right)+\boldsymbol{e}^{T}\boldsymbol{\phi}(W^{T}\boldsymbol{x}_{n}+A^{T}\boldsymbol{\lambda}_{n})\right].\label{eq:E_gensig_total}
\end{equation}

\begin{lem}
Eliminating $\boldsymbol{z}$ via maximization in (\ref{eq:Lagrangian_gensig_total})
and solving for the Lagrange parameter vector \textbf{$\boldsymbol{\lambda}$}
via minimization yield the dual objective function in (\ref{eq:E_gensig_total})
provided $u(z)$ is set to $\left(\phi^{\prime}\right)^{-1}(z)$.
\end{lem}
\begin{proof}
Since $\boldsymbol{\psi}(\boldsymbol{z})$ is continuous and differentiable,
we differentiate and solve for $\boldsymbol{z}$ in (\ref{eq:Lagrangian_gensig_total}).
This is simpler to write for the single instance Lagrangian in (\ref{eq:Lagrangian_gensig_single}).
Differentiating (\ref{eq:Lagrangian_gensig_single}) w.r.t. $\boldsymbol{z}$
and setting the result to $\boldsymbol{0}$, we get

\begin{equation}
W^{T}\boldsymbol{x}-\boldsymbol{u}(\boldsymbol{z})-\boldsymbol{z}\odot\nabla\boldsymbol{u}(\boldsymbol{z})+\nabla\boldsymbol{\phi}(\boldsymbol{u}(\boldsymbol{z}))\odot\nabla\boldsymbol{u}(\boldsymbol{z})+A^{T}\boldsymbol{\lambda}=\boldsymbol{0}.
\end{equation}
(Please see section~\ref{subsec:Notation:} on the notation used
for an explanation of these terms.) Since $\phi^{\prime}(u(z))=\phi^{\prime}((\phi^{\prime})^{-1}(z))=z$,
we have $\boldsymbol{z}=\nabla_{\boldsymbol{u}}\boldsymbol{\phi}(\boldsymbol{u}(\boldsymbol{z}))$.
From this we get

\begin{equation}
\boldsymbol{u}(\boldsymbol{z})=W^{T}\boldsymbol{x}+A^{T}\boldsymbol{\lambda}
\end{equation}
and therefore

\begin{equation}
\boldsymbol{z}^{T}W^{T}\boldsymbol{x}-\boldsymbol{z}^{T}\boldsymbol{u}(\boldsymbol{z})+\boldsymbol{z}^{T}A^{T}\boldsymbol{\lambda}=0.
\end{equation}
Substituting this in (\ref{eq:Lagrangian_gensig_total}), we get the
dual objective function (along with a minimization problem on $W$)
written for just a single training set instance to be\\
\begin{align}
E_{\mathrm{gensig}}(W,\boldsymbol{\lambda}) & =\left[-\boldsymbol{y}^{T}W^{T}\boldsymbol{x}-\boldsymbol{b}^{T}\boldsymbol{\lambda}+\boldsymbol{e}^{T}\boldsymbol{\phi}(W^{T}\boldsymbol{x}+A^{T}\boldsymbol{\lambda})\right]\nonumber \\
 & =\left[-\boldsymbol{y}^{T}\left(W^{T}\boldsymbol{x}+A^{T}\boldsymbol{\lambda}\right)+\boldsymbol{e}^{T}\boldsymbol{\phi}(W^{T}\boldsymbol{x}+A^{T}\boldsymbol{\lambda})\right]\,\,\,\,(\mathrm{since}\,A\boldsymbol{y}=\boldsymbol{b})\nonumber \\
 & =\left[-\boldsymbol{y}^{T}\boldsymbol{u}+\boldsymbol{e}^{T}\boldsymbol{\phi}(\boldsymbol{u})\right]\label{eq:Egensig_single}
\end{align}
where $\boldsymbol{u}\equiv W^{T}\boldsymbol{x}+A^{T}\boldsymbol{\lambda}$
is the linear predictor (with $\boldsymbol{z}=\boldsymbol{\sigma}(\boldsymbol{u})$
being the nonlinear predictor). Equation~(\ref{eq:Egensig_single})
when expanded to the entire training set yields (\ref{eq:E_gensig_total})
and can be rewritten in component-wise form as

\begin{align}
E_{\mathrm{gensig}}^{(\mathrm{total})}\left(W,\left\{ \boldsymbol{\lambda}_{n}\right\} \right) & =\sum_{n=1}^{N}\left[-\sum_{k=1}^{K}y_{nk}\left(\sum_{j=1}^{J}w_{jk}x_{nj}+\sum_{i=1}^{I}a_{ik}\lambda_{ni}\right)+\sum_{k=1}^{K}\phi\left(\sum_{j=1}^{J}w_{jk}x_{nj}+\sum_{i=1}^{I}a_{ik}\lambda_{ni}\right)\right]\nonumber \\
 & =\sum_{n=1}^{N}\left[-\sum_{k=1}^{K}y_{nk}u_{nk}+\sum_{k=1}^{K}\phi\left(u_{nk}\right)\right]
\end{align}
where

\begin{equation}
u_{nk}\equiv\sum_{j=1}^{J}w_{jk}x_{nj}+\sum_{i=1}^{I}a_{ik}\lambda_{ni},
\end{equation}
the component-wise linear predictor counterpart to $\boldsymbol{u}=W^{T}\boldsymbol{x}+A^{T}\boldsymbol{\lambda}$.
We have shown that we get a minimization problem on both the weights
and the Lagrange parameter vectors (one per training set instance)
in a similar manner to the case of multi-class logistic regression.
We stress that $z=\sigma(u)$ requires $\sigma(u)=\phi^{\prime}(u):$
the chosen nonlinearity and the choice of $\phi(u)$ in the loss function
in (\ref{eq:Egensig_single}) are in lockstep. This is a departure
from standard practice in present day deployments of feedforward neural
networks.

The minimization problem in (\ref{eq:E_gensig_total}) is valid for
any generalized sigmoid $\sigma(u)$ provided it is continuous, differentiable
and monotonic. We state without proof that the dual objective function
is convex w.r.t. both the top layer weights $W$ and the Lagrange
parameter vectors $\left\{ \boldsymbol{\lambda}_{n}\right\} $. However,
since we plan to deploy standard stochastic gradient descent methods
for optimization, this fact is not leveraged here. 
\end{proof}
We end by showing that the predictor \textbf{$\boldsymbol{z}$ }satisfies
the linear equality constraints. 
\begin{cor}
The nonlinear predictor $\boldsymbol{z}=\boldsymbol{\sigma}(\boldsymbol{u})=\boldsymbol{\sigma}\left(W^{T}\boldsymbol{x}+A^{T}\boldsymbol{\lambda}\right)$
satisfies the constraints $A\boldsymbol{z}=\boldsymbol{b}$ for the
optimal setting of the Lagrange parameter vector $\boldsymbol{\lambda}$. 
\end{cor}
\begin{proof}
We differentiate (\ref{eq:E_gensig_total}) w.r.t. $\boldsymbol{\lambda}_{n}$
and set the result to $\boldsymbol{0}$ to obtain

\begin{equation}
\sum_{k=1}^{K}a_{ik}\phi^{\prime}\left(\sum_{j=1}^{J}w_{jk}x_{nj}+\sum_{i^{\prime}=1}^{I}a_{i^{\prime}k}\lambda_{ni^{\prime}}\right)=b_{ni},\,\forall i\in\left\{ 1,\ldots,I\right\} 
\end{equation}
or 

\begin{equation}
A\boldsymbol{z}_{n}=\boldsymbol{b}_{n},\,\forall n
\end{equation}
at the optimal value of the Lagrange parameter vector $\boldsymbol{\lambda}_{n}$
and

\begin{equation}
z_{nk}=\phi^{\prime}\left(\sum_{j=1}^{J}w_{jk}x_{nj}+\sum_{i=1}^{I}a_{ik}\lambda_{ni}\right)=\sigma\left(\sum_{j=1}^{J}w_{jk}x_{nj}+\sum_{i=1}^{I}a_{ik}\lambda_{ni}\right),\,\forall n,k,
\end{equation}
and we are done.
\end{proof}

\section{Applications}\label{sec:Applications}

\subsection{Synthetic Equality Constraints}\label{subsec:Synthetic-Equality-Constraints}

We now demonstrate the incorporation of equality constraints through
a simple example. To this end, we construct a feedforward network
for the XOR problem using 4 fully connected layers. The last layer
(fc4) is a linear classifier with a single output whereas the previous
layer (fc3) converts the fc2 outputs into new features which are easier
to classify than the original coordinates. Consequently, we visualize
the fc3 weights as enacting linear constraints on the fc2 outputs
in the form $A\boldsymbol{x}_{2}=\boldsymbol{b}$ where $A$ is the
set of weights corresponding to the fc3 layer with $\boldsymbol{x}_{2}$
being the output of the fc2 layer and $\boldsymbol{b}$ being the
set of final XOR features. We train this network in PyTorch using
a standard approach with 2000 training set samples in $\mathbb{R}^{2}$.
The network architecture is as follows: 10 weights in fc1 (2 units
to 5 units), 50 weights in fc2 (5 units to 10 units), 40 weights in
fc3 (10 units to 4 units) and 4 weights in fc4 (4 units to 1 unit).
Each set of weights in each layer is accompanied by bias weights as
well. After training, we extract the weights corresponding to $A$
(the weights of the fc3 layer) and store them.

Next, we set up our constraint satisfaction network. The goal of this
network is to begin with fresh inputs in $\mathbb{R}^{2}$ and map
them to the outputs of the previous network's fc2 layer ($\boldsymbol{x}_{2}$
which are the new targets $\boldsymbol{y}$) while satisfying the
constraints ($A\boldsymbol{x}_{2}=\boldsymbol{b})$. This is explained
in greater detail. Assume we have new inputs $\boldsymbol{x}_{0}^{(\mathrm{new})}$.
These are sent into the XOR network to produce the output of the fc2
layer $\boldsymbol{x}_{2}$ which we now call $\boldsymbol{y}$, the
new targets. But, we also know that the targets $\boldsymbol{y}$
satisfy additional constraints in the form of the XOR network's fc3
layer. Therefore, if we construct a new network with $\boldsymbol{y}$
being the target of the fc2 layer then the new network (henceforth
XOR Lagrange) has constraints $A\boldsymbol{x}_{2}^{(\mathrm{new})}=A\boldsymbol{y}=A\boldsymbol{x}_{2}$
which need to be enforced. Since the XOR fc3 layer maps 10 units to
4 units, there are 4 constraints for each $\boldsymbol{x}_{2}^{(\mathrm{new})}$
in XOR Lagrange. To enforce these, we set up a custom layer which
has additional fixed weights (mapping 10 units to 4 units) in the
form of $A$ and a set of Lagrange parameters (4 per training or test
pattern) which need to be obtained. We train the new XOR Lagrange
network in the same way as the XOR network but with one crucial difference.
In each epoch and for each batch, we use a constraint satisfaction
network to solve for the 4 Lagrange parameters for each pattern in
each batch. A standard gradient descent approach with a step size
solver is used. We use a threshold of $5\times10^{-3}$ for the average
change in the loss function w.r.t. $\boldsymbol{\lambda}$ (the Lagrange
parameters)---essentially a standard BCE loss---and a threshold
of $1.0\times10^{-6}$ for the step-size parameter (where the search
begins with $\alpha$ the step-size parameter set to 0.1 and then
halved each time if too high). If convergence is not achieved, we
stop after 1000 iterations of constraint satisfaction. Note that we
do not need to completely solve for constraint satisfaction in each
epoch but have chosen to do so in this first work. 

XOR Lagrange was executed 1000 times on a fresh input (with 2000 patterns)
with the results shown in Figure~\ref{fig:XOR-Lagrange}. We discovered
a very common pattern in the optimization. Initially for every fresh
input $\boldsymbol{x}_{0}^{(\mathrm{new})}$, $\boldsymbol{x}_{2}^{(\mathrm{new})}$
is far away from $\boldsymbol{y}$ and the constraint satisfaction
solver can take many iterations to obtain $A\boldsymbol{x}_{2}^{(\mathrm{new})}\approx A\boldsymbol{y}$
(and note that both $\boldsymbol{x}_{2}^{(\mathrm{new})}$ and $\boldsymbol{y}$
are the outputs of sigmoids). Then, as optimization proceeds, the
number of constraint satisfaction iterations steadily decreases until
$\boldsymbol{x}_{2}^{(\mathrm{new})}\approx\boldsymbol{y}$ (and we
always execute for 100 epochs). In the figure, we depict this progression
for the maximum, minimum and mean number of iterations. The loss function
in each epoch is also shown (in lockstep) for the sake of comparison.
While the mean number of constraint satisfaction iterations trends
lower, we do see instances of large numbers of iterations in the maximum
case. This suggests that we need to explore more sophisticated optimization
schemes in the future. 

We conclude with an initial foray into improved constraint satisfaction.
In the above set of experiments, we did not propagate the Lagrange
parameters across epochs. This causes the constraint satisfaction
subsystem to repeatedly solve for the Lagrange parameters from scratch
which is wasteful since we have a global optimization problem on both
$W$ and $\left\{ \boldsymbol{\lambda}_{n}\right\} $. To amend this,
we treat $\left\{ \boldsymbol{\lambda}_{n}\right\} $ as state variables
and initialize the constraint satisfaction search in each epoch using
the results from the previous one. The only parameter value change
from the above is a threshold of $1\times10^{-3}$ for the average
change in the loss function w.r.t. the Lagrange parameters with this
lower value chosen due to the improved constraint satisfaction subsystem.
Despite this tighter threshold, the total number of iterations for
constraint satisfaction is lowered as can be seen in Figure~\ref{fig:XOR-Lagrange_history}.
The figure also depicts the evolution toward convergence of a suitable
Lagrange parameter norm $\|\boldsymbol{\lambda}\|$.

\section{Discussion}\label{sec:Discussion}

The motivation behind the work is to endow neural networks with constraint
satisfaction capabilities. We have shown this can be accomplished
by appealing to a saddle-point principle---the difference of Bregman
divergences---with constraints added to a maximization problem which
is carefully defined on a new set of auxiliary predictor variables.
When these predictors are eliminated from the Lagrangian, we obtain
a dual minimization problem on the Lagrange parameters (one per training
set instance). The price of admission for constraint satisfaction
therefore is an expansion of the penultimate hidden layer with additional
Lagrange parameter units (and fixed weights corresponding to the constraints).
Solutions for the Lagrange parameters are required in addition to
weight learning with further optimization needed during testing. The
advantage of this formulation is that we have a single minimization
problem (as opposed to a saddle-point problem) with the unusual interpretation
of Lagrange parameters as hidden units. Immediate future work will
contend not only with linear inequality constraints (perhaps achieved
via non-negativity constraints on the Lagrange parameters) but nonlinear
equality constraints which may require modifications of the basic
framework. We envisage a host of applications stemming from this work,
driven by domain scientists who can now attempt to impose more structure
during learning using application-specific constraints. 
\begin{center}
\begin{figure}
\begin{centering}
\includegraphics[width=1\textwidth]{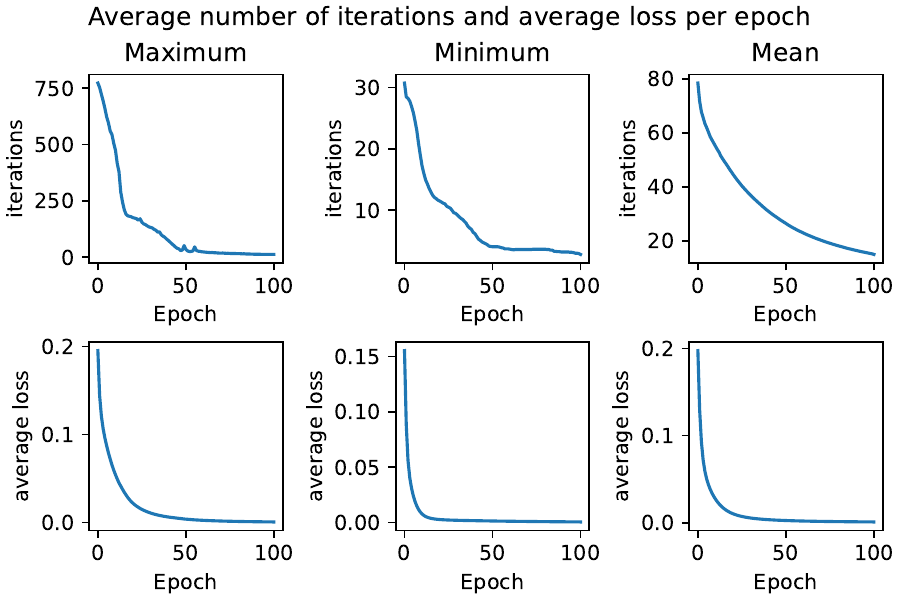}
\par\end{centering}
\caption{\textbf{XOR Lagrange}. In each subplot, we show training progress
with constraint satisfaction. In each top subplot, the constraint
satisfaction progress is depicted with its loss counterpart proceeding
in lockstep below. The maximum, minimum and mean iterations (for 1000
runs) are shown in the left, middle and right subplots respectively.}\label{fig:XOR-Lagrange}
\end{figure}
\par\end{center}

\begin{center}
\begin{figure}
\begin{centering}
\includegraphics[width=1\textwidth]{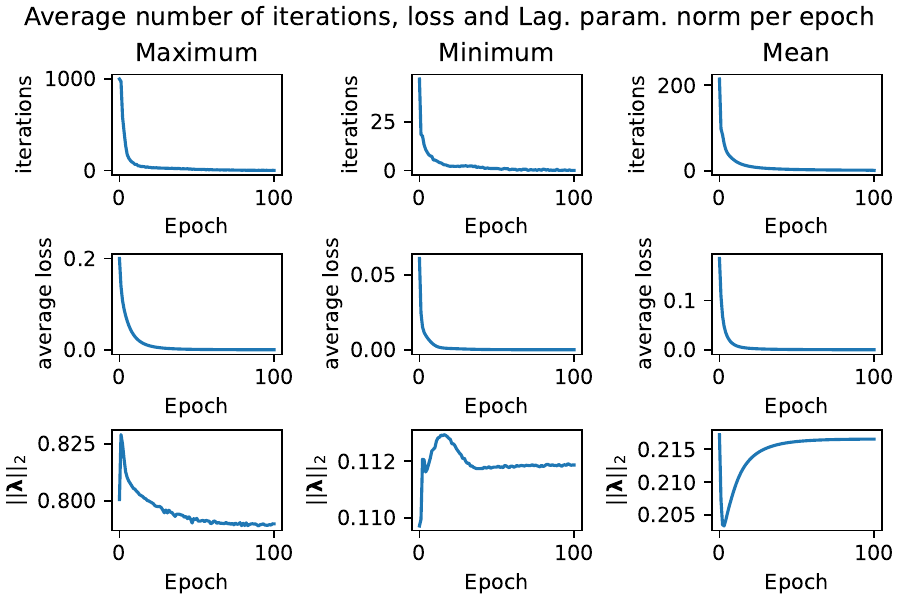}
\par\end{centering}
\caption{\textbf{XOR Lagrange with history}: In each top subplot, the constraint
satisfaction progress is depicted with its loss counterpart proceeding
in lockstep immediately below. The bottom set of plots show the evolution
of the Lagrange parameters via their norm $\|\boldsymbol{\lambda}\|$.
The maximum, minimum and mean iterations (for 1000 runs) are shown
in the left, middle and right subplots respectively.}\label{fig:XOR-Lagrange_history}
\end{figure}
\par\end{center}

\bibliographystyle{apalike}
\bibliography{constraints}

\end{document}